\pdfoutput=1
\documentclass[letterpaper, 10 pt, conference]{ieeeconf}  

\IEEEoverridecommandlockouts                              

\let\labelindent\relax

\usepackage{graphicx}
\usepackage{url}            
\usepackage{amsfonts}       
\usepackage{nicefrac}       
\usepackage{color}
\usepackage{amssymb} 
\usepackage{amsmath} 
\usepackage{dsfont}
\usepackage{soul}
\usepackage{subcaption}
\usepackage{bbm}
\usepackage{enumitem}
\usepackage[ruled]{algorithm2e}

\setlist[itemize]{leftmargin=*}
\setlist[enumerate]{leftmargin=*}

\def \1{\mathbf{1}}
\def \0{\mathbf{0}}
\def \R{\mathbb R}

\def \hD{\widehat{D}}

\def \b1{\boldsymbol{1}}
\def \T{{\cal T}}

\def \S{{\cal S}}

\newtheorem{theorem}{Theorem}[section]
\newtheorem{corollary}{Corollary}[theorem]
\newtheorem{lemma}[theorem]{Lemma}

\newcommand{\supp}{\text{supp}}

\makeatletter
\newcommand{\removelatexerror}{\let\@latex@error\@gobble}
\makeatother
\title{\LARGE \bf
If it ain't broke, don't fix it: Sparse metric repair*
}

\author{Anna C.~Gilbert$^{1}$ and Lalit Jain$^{2}$
\thanks{*ACG was supported by National Science Foundation grants CCF-1161233 and CIF-0910765.}
\thanks{$^{1}$Anna C.~Gilbert is with the Department of Mathematics,
        University of Michigan, Ann Arbor, MI
        {\tt\small annacg@umich.edu}}%
\thanks{$^{2}$Lalit Jain is with the Department of Mathematics,
        University of Michigan, Ann Arbor, MI
        {\tt\small lalitj@umich.edu}}%
}

\begin{document}

\maketitle
\thispagestyle{empty}
\pagestyle{empty}

\begin{abstract}

Many modern data-intensive computational problems either require, or benefit from distance or similarity data that adhere to a metric. The algorithms run faster or have better performance guarantees. Unfortunately, in real applications, the data are messy and values are noisy. The distances between the data points are far from satisfying a metric. Indeed, there are a number of different algorithms for finding the closest set of distances to the given ones that also satisfy a metric (sometimes with the extra condition of being Euclidean). These algorithms can have unintended consequences; they can change a large number of the original data points, and alter many other features of the data. 

The goal of \emph{sparse metric repair} is to make as few changes as possible to the original data set or underlying distances so as to ensure the resulting distances satisfy the properties of a metric. In other words, we seek to minimize the sparsity (or the $\ell_0$ ``norm'') of the changes we make to the distances subject to the new distances satisfying a metric. We give three different combinatorial algorithms to repair a metric sparsely. In one setting the algorithm is guaranteed to return the sparsest solution and in the other settings, the algorithms repair the metric. Without prior information, the algorithms run in time proportional to the cube of the number of input data points and, with prior information we can reduce the running time considerably. 

\end{abstract}

\section{INTRODUCTION}

For many data analysis and processing tasks, we assume that the data upon which we run the various algorithms are points in a metric space; that is, the distances between the data points satisfy the properties of a (semi-)metric. We do not necessarily assume that the distances are Euclidean, nor that we have distinct points, but many of these algorithms do require the distances to satisfy the triangle inequality. In machine learning, for example, learning the metric from which a set of points is drawn is critical for tasks such as clustering~\cite{Xing2002}. For kernelized learning algorithms, we learn the kernel (rather than the more general metric) by learning Euclidean distances~\cite{chen2009learning}. For clustering data, ensuring that all the distances satisfy the triangle inequality constraints guarantees the algorithms run efficiently and produce correct clusters. Indeed, many approximation algorithms for standard theoretical computer science tasks have better approximation guarantees and highly efficient solutions (e.g., sublinear in the input size) when the input is from a metric space. See Indyk's survey of sublinear approximation algorithms~\cite{Indyk:1999:STA:301250.301366}  for examples such as the Traveling Salesman Problem, Maximum Spanning Tree, and $k$-median. Finally, in DNA sequence alignment problems, we build amino acid substitution matrices~\cite{anfinsen2005making} (known as Point Accepted Mutation and Block Substitution Matrices) to speed up pattern matching and alignment tasks, under the assumption that the matrices represent distances (from a metric) amongst amino acids.

Unfortunately, data are messy, physical measurements are noisy, values are missing, and data sets are far from perfect metrics. Before we can employ any of these efficient data analysis algorithms for critical tasks in machine learning, bioinformatics, or graph analysis, we must \emph{repair the metric}. We must adjust the data points so as to ensure their distances satisfy the properties of a metric. In the psychometric literature, there is a long history of work on the implications of data that fail to satisfy a metric, as well as many algorithms to ``fix'' the data so that the distances are Euclidean. See~\cite{laub2007inducing} and the references therein (especially the work of Shepard from the 1960s).
For clustering, Baraty, et al.~\cite{baraty2011impact} discuss what impact the failure of distances to adhere to a metric has upon clustering algorithms and they use rectifiers to continuously modify all distances to obtain a metric. Brickell, et al.~\cite{Brickell} were the first to refer to this problem specifically and to define the metric nearness or repair problem precisely: given a set of points or, equivalently, a matrix of distances, find the closest (in an appropriate norm) matrix of distances that satisfies metric properties. They devise several algorithms based upon convex optimization, as well as several combinatorial algorithms, closely related to all pairs shortest path computations on undirected, weighted graphs. 

All of the previously developed algorithms can (and, in practice, do) adjust all of the distances in their repairs. Biswas and Jacobs~\cite{biswas2015efficient} observe, however, that in many applications, more than 99\% of the triangle inequalities are satisfied. In other words, these algorithms make unnecessary and potentially deleterious changes in the original data. While the algorithms may fix the metric, they may break or change in unforeseen ways other aspects of the data set. Which leads us to advise (colloquially), ``If it ain't broke, don't fix it.'' Our precise mathematical goal is to make as few changes as possible to the original data set or underlying distances so as to ensure the resulting distances satisfy the properties of a metric. In other words, we seek to minimize the sparsity (or the $\ell_0$ ``norm'') of the changes we make to the distances among data points subject to the new distances satisfying a metric. We refer to this problem as \emph{sparse metric repair}. Note that we do not insist that the repaired metric be Euclidean. It could be from a manifold or class based from the output of a classifier. Enforcing a triangle inequality is the least we can ask from similarity or distance data.

We formulate three different repair scenarios, decrease only metric repair, increase only metric repair, and the general case. Each is structurally different from the others and, as a result, have different algorithmic solutions and guarantees. The decrease only metric repair problem is closely related to the all pairs shortest path problem (first observed by Brickell, et al.~\cite{Brickell}) and we use these observations to formulate a combinatorial algorithm that solves the sparse metric repair problem in time that is cubic in the number of data points. We give similar results for the other two scenarios, although the other two are much more difficult to analyze (as we show). The common feature of all of our algorithms is that without prior knowledge of which triangles are broken or which triples of distances fail to satisfy the triangle inequality, the running time is cubic. With such knowledge, the algorithms run considerably faster as they fix only what is broken.

We begin the paper with a precise statement of the sparse metric repair problem and a discussion of several restricted settings in which to study the problem. In Section~\ref{sec:algos}, we present algorithmic results for the three settings, along with some probabilistic analysis of a random data model. We then analyze the experimental performance of our algorithms in Section~\ref{sec:expts} on a range of problem types. We also compare our combinatorial algorithms with those from a more standard convex relaxation approach. 

\section{Preliminaries and Notation}

\newcommand{\symn}{\text{Sym}_n(\R_{\geq 0})}
We denote the subspace of positive real symmetric $n\times n$ matrices by $\symn$. A matrix $D\in \symn$ is said to be \textit{metric}, or equivalently a \textit{distance matrix}, if
\begin{align*}
    &D_{ii} = 0 \text{ for }1\leq i\leq n \\
    &D_{ij} \leq D_{ik}+D_{jk} \text { for }1\leq i,j,k\leq n.
\end{align*} 
Note that our notion of metric is often known as a semi-metric. In particular, we do not insist that $D_{ij} = 0$ iff $i = j$.  

In some situations, we wish to indicate that one matrix is element-wise less than another matrix and we write $A \preceq B$ to signify that each element of $A$ is less than or equal to its corresponding element in $B$. For example, one of the conditions a distance matrix satisfies is $0 \preceq D$.

We will refer to the complete (undirected) weighted graph on $n$ vertices with edge weights given by $D$ as $K_n(D)$. Our triangles are labeled according to their associated triangle inequality. We say that the triangle inequality corresponding to \textit{triangle} $ijk$ for a distance matrix $D$ is $D_{ij} \leq D_{ik} + D_{jk}$. Note that triangle $ijk = jik$, but $ijk\neq ikj\neq jki$. In particular there are $n(n-1)(n-2)/2$ total triangles. We say that $D_{ij}$ occurs on the left hand side of the above triangle inequality, and $D_{ik}$ and $D_{jk}$ occur on the right. The triangle $ijk$ and the associated triangle inequality are used interchangeably. 

Finally, we let $\T(D)$ denote the set of all $ijk$ with $D_{ij} > D_{ik} +D_{jk}$ for a $D\in \symn$. We refer to $\T(D)$ as the set of \textit{broken triangles}.\footnote{In the literature on shortest path algorithms, these are referred to as ``negative triangles.''} 

\section{Sparse Metric Repair Problem Statement}
Given a perturbed or corrupted distance matrix, $D'\in \symn$, the goal of \textit{sparse metric repair} is to find a perturbation $P$ so that $D = D'+P$ is metric, $P$ is sparse as possible, and $P$ is contained in a specified set $S$. 
\begin{equation}
\begin{aligned}
  &\text{minimize} & & \|P\|_0\\
  &\text{subject to} & & D'+P \text{ is metric}\\   
  &\text{and} & & P\in S\subset \mathbb{R}^{n\times n}
\end{aligned}
\end{equation}
As we will see in the following, assumptions on the constraint set $\S$ play a large role in the types of algorithms, the complexity of the algorithms,  and the types of guarantees that we can give for metric repair. We consider three specific cases:

\noindent{\textbf{Decrease Only Metric Repair (DOMR)}}: In this case we assume that $P \preceq 0$, i.e. we are interested in only decreasing the distances of $D'$ to ensure metricity.

\noindent{\textbf{Increase Only Metric Repair (IOMR)}}: In this case we assume that $0 \preceq P$, i.e. we are interested in only increasing the distances of $D'$ to ensure metricity.

\noindent{\textbf{General Metric Repair}}: We do not place any restrictions on $P$.


In general, there are many ways to repair $D'$ to assure that $D'+P$ is a distance matrix. For example, as pointed out in the multidimensional scaling literature~\cite{GowerLegendre86}, if $c = \max_{i,j,k} D'_{ik} + D'_{jk} - D'_{ij}$ then $D'+ c$ is automatically a metric. Also, note that uniqueness of $P$ is not guaranteed. For example in a triangle with side lengths $1, 2, 7$, we can increase either of the two smaller sides, or decrease the larger side to make it metric. In the first example, the repair is far from sparse as it adds one value to all the distances and in the second example, there are two different repairs, one changes fewer distances than the other. Indeed, there are many cases in which we know that $D' = D + \Delta$ where $\Delta$ is sparse and so we seek to change or repair as few distances as possible so as to (approximately) recover $D$.

We note that this is a challenging problem. Even even if we know the support of $\Delta$ exactly, in general, we can not hope to recover $D$ exactly from $D'$. Indeed, there could be multiple repairs $P$ of the same sparsity level so that $D'-P$ is a distance matrix. 

 By the nature of the problem, the repair $P$ only needs to be supported on edges appearing in any triangle in $\T(D')$, or in any triangle involving a distance in $\T(D')$. For any pair $i,j$, there are roughly $2n$ triangle inequalities involving $i,j$, so the support of $P$ can be as large as $O(|\T(D')|n)$. Depending on the types of assumptions or restrictions on $P$, this bound can be a very loose upper bound.

In this paper, we focus on two main types of algorithms. In the next section, we describe algorithms for sparse DOMR, sparse IOMR, and the general case that are fundamentally combinatorial in nature and that aim to repair the metric directly. We also present convex relaxation methods based on minimizing the $\ell_1$ norm, a technique that is well-established in the sparse analysis literature. Finally, we discuss the case of random data and provide experiments contrasting combinatorial and convex approaches.


\section{Algorithms for Metric Repair} \label{sec:algos}
\subsection{Decrease Only Metric Repair (DOMR)} \label{sec:DOMR}
We begin by considering the case when $P\leq 0$, so $\S$ is the subspace of all positive symmetric matrices. This is also known as the \textit{Decrease Only Metric Repair} problem (DOMR). A few initial remarks are necessary. First, assume that triangle $ijk$ is broken; i.e. $D_{ij}' \geq D'_{ik}+D'_{jk}$. Decreasing $D'_{ik}$ or $D'_{jk}$ will not fix this triangle inequality. Hence, we know that any distances $D'_{i,j}$ that occur on the left hand side of a broken triangle inequality necessarily are in the support of $P$. 

Second, the decrease only metric repair problem is closely related to that of all pairs shortest path (APSP) on a graph and this relationship is easy to see. Because $D = D'+P$ is metric, by definition, every $D_{ij} \leq \min_{k} (D_{ik}+D_{jk})$. This has a natural interpretation in terms of shortest paths on a graph---the shortest path from $i$ to $j$ is given by the distance $D_{ij}$ from $i$ to $j$.  The relationship between DOMR and APSP was first observed by Brickell et al.~\cite{Brickell}. In fact they show that the problems are equivalent; a decrease only solution gives rise to an APSP solution and vice versa.

In particular, this implies that any algorithm for APSP can be reused for DOMR. There are many algorithms for solving APSP; to name a few, the Floyd-Warshall algorithm (presented in Algorithm  \ref{alg-DOMR}), repeated use of single source shortest path methods such as Dijkstra's algorithm, and recursive methods based on matrix multiplication in the tropical semiring (see Williams and Williams~\cite{WilliamsWilliams2010} for a full discussion of the relationship amongst these problems). Brickell, et al.~\cite{Brickell} also provide a primal-dual algorithm (similar to the method using Dijkstra's algorithm) for APSP. All of the above algorithms have a run time of $O(n^3)$ in the worst case. We will discuss the use of subcubic algorithms for APSP in broken triangle detection further in Section \ref{sec-brokentriangles}. 

\begin{algorithm}
    \SetAlgoLined
    \caption{Floyd-Warshall for DOMR}

\KwIn{Corrupted $n\times n$ distance matrix $D'$}
\KwResult{Perturbation $P$}
 $\widehat{D} = D'$\\
 \For{$k=1$ \KwTo $n$}{
    \For{$i=1$ \KwTo $n$}{
        \For{$j=1$ \KwTo $i-1$}{
            \lIf{$\widehat{D}_{ij}\geq \widehat{D}_{ik}+\widehat{D}_{kj}$}{
            $\widehat{D}_{ij} = \widehat{D}_{ik}+\widehat{D}_{kj}$
            }
        } 
    }
 }
$P = \widehat{D}-D'$
\label{alg-DOMR}
\end{algorithm}

To summarize Brickell, et al.'s observation:
\begin{lemma}(\textbf{Brickell, et al.~\cite{Brickell}})
\label{lem:APSP}
Let $\widehat{D}$ be any solution to DOMR; i.e. $\widehat{D}$ is element-wise less than $D'$. In addition, let $D^A$ be the all pairs shortest path solution for $K_n(D')$, the complete graph with edge weights given by $D'$. Then, $\hat{D}$ is element-wise smaller than $D^A$, $\hat{D} \preceq D^A$. 
\end{lemma}

This lemma has several immediate implications for finding sparse solutions. Let the perturbation associated to $D^A$ be $P^A := D^A - D'$. Consider a decomposition $D' = \widehat{D}+\Delta$ with $\widehat{D}$ metric and $\Delta\geq 0$. By Lemma~\ref{lem:APSP}, $\widehat{D}\preceq D^A\preceq D'$, so if $\Delta_{ij} = 0$, then $D_{ij}^A = D'_{ij}$ which implies $(P^A)_{ij} = 0$. In particular, this implies that $\supp(P^A)\subset \supp(\Delta)$. We summarize these observations for (sparse) DOMR in two corollaries. 
\begin{corollary}
The perturbation $P^A := D^A - D$ is a sparsest possible decrease only metric repair solution.
\label{cor:sparseAPSP}
\end{corollary}

\begin{corollary}
The lemma also immediately implies that $D^A$ is the solution to
\[
    \text{argmin}_{D \preceq D'}\|D-D'\|_{p}
\] 
for any $\ell_p$ norm. In particular, if we set $p=1$, we see that $D^A$ is the minimal $\ell_1$ norm solution. 
\end{corollary}

The algorithmic implication of Corollary~\ref{cor:sparseAPSP} is that the only distances that must be repaired (or, more accurately, adjusted) are those distances that were corrupted to begin with. In particular, this means that the only triangles that should be modified in the course of the algorithm are triangles that are already broken. The Floyd-Warshall algorithm presented in Algorithm \ref{alg-DOMR} is particularly amenable to including prior information about broken triangles or distances. In particular, we only have to consider triangles $ijk$ that are in $\T(D')$, or triangles that could potentially break if we decrease a distance $D_{ij}$; i.e. potentially breaking a triangle inequality of the form $D'_{il}\leq D'_{ij} + D'_{jl}$. (We refer to this extension as $\T'$ in the pseudo-code below.) In particular, this implies a $O(|\T(D')|n)$ algorithm for DOMR.  In general, if our perturbation $P$ is sparse, we expect that the number of involved triangles will be small as well. If we perturb $k$, distances, then we will have to consider $O(kn)$ triangles---far fewer from the $n^3$ required by the naive Floyd-Warshall algorithm. Our modified algorithm is given in Algorithm \ref{alg-DOMRPrior}.
\begin{algorithm}
\SetAlgoLined
\caption{Floyd-Warshall with Prior Information}

\KwIn{Corrupted $n\times n$ distance matrix $D'$, $\T(D')$ ordered lexicographically}
\KwResult{Perturbation $P$}
 $\widehat{D} = D'$\\
$\T'$ extension of $\T$; all $lij$ for $ijk\in \T$\\
\For{$t=ijk$ in $\T'$}{
   \lIf{$\widehat{D}_{ij} > \widehat{D}_{ik}+\widehat{D}_{kj}$}{
            $\widehat{D}_{ij} = \widehat{D}_{ik}+\widehat{D}_{kj}$
            }
 }
 $P =\widehat{D} - D'$ 
\label{alg-DOMRPrior}
\end{algorithm}

\begin{theorem}
Given a corrupted distance matrix $D'$ and a list $\T(D')$ of broken triangles, ordered lexicographically, Floyd-Warshall with Prior Information returns the sparsest repair $P$ in time $O(|\T(D')|n)$.
\end{theorem}


\subsection{Increase Only Metric Repair} \label{sec:IOMR}
 Next, we consider the Increase Only Metric Repair problem (IOMR); that is, we assume a repair of the form $0 \preceq P$. Though superficially similar to DOMR, IOMR is significantly more difficult and far fewer guarantees can be given. 

Unlike the DOMR case, it is not possible to detect which distances were perturbed from a broken triangle inequality. Indeed, if $D_{ij} \geq D_{ik} + D_{jk}$, then either $D_{ik}$ or $D_{jk}$ could have been initially perturbed (decreasing $D_ij$ would not have broken the triangle inequality). In addition, it is easy to see by an explicit example that the $\ell_1$ minimizing IOMR solution does not give any sparsity guarantees. Consider our previous example: a triangle with side lengths $1, 2, 7$, we can increase either of the two smaller sides with a total increase of 4 or decrease the larger side by 4 to make it metric. There are infinite solutions with a total $\ell_1$ norm of 4 but only three solution which changes only one distance; i.e., only a few sparse solutions out of multiple solutions with minimal $\ell_1$ norm.

We are, however, encouraged by the success of Floyd-Warshall in the DOMR problem and we provide Algorithm~\ref{alg-IOMR} that updates $D_{ik}$ with $D_{ij} - D_{jk}$ whenever $ijk$ is broken.

\begin{algorithm}
\caption{Increase Only Metric Repair}

\KwIn{Corrupted $n\times n$ distance matrix $D'$}
\KwResult{Perturbation $P$}
 $\widehat{D} = D$\\
 \For{$k=1$ \KwTo $n$}{
    \For{$i=1$ \KwTo $n$}{
        \For{$j=1$ \KwTo $i-1$}{
            \If{$\widehat{D}_{ij} > \widehat{D}_{ik}+\widehat{D}_{kj}$}{
                $\widehat{D}_{ik} =  \widehat{D}_{ij}-\widehat{D}_{kj}$
            }
        } 
    }
 }
$P = \widehat{D} - D'$
\label{alg-IOMR}
\end{algorithm}

\begin{lemma}
Given a corrupted distance matrix $D'$, Algorithm \ref{alg-IOMR} for IOMR returns a positive perturbation $P$ so that $D'+P$ is metric in time $O(n^3)$.
\end{lemma}
\begin{proof}
At each step of the inner loop, a triangle is fixed. It suffices to show that these steps are not destructive; i.e. after the end of iteration $k$ of the outer loop, by increasing $\hD_{ik}$, no triangle inequality $ikl$, for any triangle with $l < k$, breaks. If $l>k$, we will visit this triangle in the future and repair it.

Choose $j$ to be the largest index so that we update triangle $ijk$ and assume updating $\hD_{ik}$ breaks triangle $ikl$ as a result. At the end of the inner most loop,  $\hD_{ik} = \hD_{ij} - \hD_{kj}$. This implies,
\begin{align*}
   &\hD_{ik} > \hD_{il} + \hD_{lk}
   \Rightarrow \hD_{ij} - \hD_{kj} > \hD_{il} + \hD_{lk}\\
   &\Rightarrow \hD_{ij} > \hD_{il}+\hD_{lk}+\hD_{kj}.
\end{align*}
Since $l < k$, triangle $ijl$ is fixed, and so $\hD_{ij} \leq  \hD_{il}+\hD_{lj}$. Plugging this in as an upper bound for $\hD_{ij}$, we see that
\[
    \hD_{il}+\hD_{lj} \geq \hD_{ij} >\hD_{il}+\hD_{lk}+\hD_{kj},
\]
which implies $\hD_{lj} > \hD_{lk}+\hD_{kj}$. In particular, the above manipulation shows if $ikl$ is broken, then $ljk$ must be as well. 

Now, if $j<l<i$ or $l<j<i$, then we have already visited triangle $ljk$ (on this iteration of the outside loop) and fixed it. In addition, updating $\hD_{ik}$ will not affect triangle $ljk$, or its equivalent $jlk$ in the latter case. This implies that triangle $ikl$ cannot have been broken. So, we can assume that $l > i > j$, and the algorithm has fixed triangle $ljk$ by setting $\hD_{lk} = \hD_{lj} - \hD_{jk}$. In this case, 
\begin{align*}
     \hD_{il} + \hD_{lk} &=  \hD_{il} + \hD_{lj} - \hD_{jk} 
     \geq \hD_{ij} - \hD_{jk} = \hD_{ik} 
\end{align*}
The last inequality holds since $l<k$, $ijl$ is not broken. Thus, we repair triangle $ikl$. 

\end{proof}

In general, we do not guarantee that Algorithm \ref{alg-IOMR} provides the sparsest possible perturbation. Indeed, if we increase a distance $\hD_{ik}$ that was not initally perturbed, many triangles of the form $\hD_{ik} \leq \hD_{il} + \hD_{kl}$ could break as a result and we update many distances $\hD_{il}$. As demonstrated in Section~\ref{sec:expts}, this algorithm, nevertheless, tends to provide sparse solutions on average. As in Algorithm~\ref{alg-DOMR}, we can reduce the runtime from $O(n^3)$ to to $O(|\T(D')|n)$ by using prior information. We provide Algorithm~\ref{alg-IOMR} as a natural counterpoint to Algorithm~\ref{alg-DOMR} as it solves the increase only metric repair problem.   

Now we turn to a method for IOMR that utilizes prior information. Assume that $D'=D+\Delta$ and we have access to an oracle $Q$ that knows precisely the support of $\Delta$. In theory, we could use a linear program to find the exact values of $P$ using $Q$. Instead we use a different method that tracks bounds on the distances provided by the triangle inequalities. 

The oracle $Q$ is a 0-1 matrix with $Q_{ij} = 1$ if $i,j$ is in $P$ and is $0$ otherwise. Our goal is to find an IOMR solution $D' \preceq \hD$ so that $D'-D$ is supported only on entries where $Q_{ij} = 1$. Algorithm~\ref{alg-IOMROracle} solves this problem, by using $Q$ to turn the IOMR problem into a DOMR problem and then running an APSP routine. The algorithm first finds an upper bound $U_{ij}$ for each $D_{ij}$ by computing 
\[ 
    \min_{Q_{ik} = 0, Q_{jk} = 0}(D_{ik}+D_{jk}).
\]
It then returns the APSP matrix for $U$ via the Floyd-Warshall algorithm. Note that each distance is replaced with its upper bound, thus guaranteeing that we only increase the distances. Lemma \ref{lem-IOMROracle} proves the validity of this algorithm.

In general we cannot expect to know $Q$, but there are easy heuristic algorithms for constructing a close approximation. For example:
\begin{itemize}
    \item We can scan over all broken triangles and set $Q(i,j)=1$ for distances that occur the most on the right hand side of a triangle inequality. Alternatively, we can scan over all broken triangles and take a subset of distances, so that every triangle contains at least one of these distances. This is similar to the standard set-cover problem which has many fast heuristics.
    \item Finally, we can run an APSP method on $D'$ and look at the edges that shortest paths tend to be routed through the most. It is very reasonable to assume that such edges have been decreased.
\end{itemize}

We will demonstrate the efficacy of the first method in Section~\ref{sec:expts}. 


 


\begin{lemma}
\label{lem-IOMROracle}
Assume that the corrupted distance matrix is of the form $D' = D+P$, with $D$ metric and that we are given access to an oracle $Q$ where $Q_{ij} = 1$ iff $P_{ij} \neq 0$. In addition, assume that for each distance $ij$ with $Q_{ij} = 1$, we can find a $k$ so that $Q_{ik} = 0$ and $Q_{jk} = 0$;  i.e., each distance $ij$ is contained in at least one triangle $ijk$ that is not broken. Then Algorithm \ref{alg-IOMROracle} will return an IOMR solution supported only on entries in $P$ and does so in time $O(n^3)$. 
\end{lemma}
\begin{proof}
It suffices to show that $\hat{U}_{ij} = D_{ij}$ whenever $Q_{ij} = 0$ by the end of the algorithm. We claim that after the Upper-Bound step and before the APSP step, $D'_{ij} \leq D_{ij} \leq U_{ij}$ for all $ij$; i.e. $U$ is an upper bound for $D$. Indeed, if $Q_{ik} = Q_{jk} = 0$, then $D'_{ik} = D_{ik}, D'_{jk} = D_{jk}$. And, by the algorithm, 
\[
    U_{ij} = \min_{k: Q_{ik} = Q_{jk} = 0}(D'_{ik} + D'_{jk}) = 
    \min_{k: Q_{ik} = Q_{jk} = 0}(D_{ik} + D_{jk}).
\] 
Since each term in the minimum is an upper bound for $D_{ij}$, by metricity of $D$, $U_{ij}$ is also an upper bound for $D_{ij}$. Let us focus on the case when $Q_{ij} = 0$, i.e. edge $ij$ has not been perturbed,  then $U_{ij} = D_{ij} = D'_{ij}$. In this case every triangle $ijk$ in $U$ is correct by the end of the upper bound step. That is,
\[
    U_{ij} =  D_{ij} \leq D_{ik}+D_{jk}\leq U_{ik} + U_{jk}.
\]

In particular, we see that this implies that the shortest path from $i$ to $j$ in $K_n(U)$ is the edge $ij$. Hence after the APSP step, $\hat{U}_{ij} = U_{ij} = D_{ij}$.




\end{proof}

\begin{algorithm}
\SetAlgoLined
\caption{Oracle Based Increase Only Metric Repair}

\KwIn{Corrupted $n\times n$ distance matrix $D'$, $Q$ oracle}
\KwResult{Perturbation $P$}
Let $U$ be an $n\times n$ symmetric matrix initialized to $\infty$\\
 // Upper-Bound step\\
 \For{$k=1$ \KwTo $n$}{
    \For{$i=1$ \KwTo $n$}{
        \For{$j=1$ \KwTo $i-1$}{
            \If{ $D'_{ij}\leq D'_{ik}+D'_{kj}$}{
                \If{$Q_{ik} = 0$ \textbf{ and } $Q_{jk} = 0$}{
                    $U_{ij} = \min(U_{ij}, D'_{ik}+D'_{jk})$
                }
            } 
        } 
    }
 }
 // APSP-step\\
$\widehat{U} = U$\\
 \For{$k=1$ \KwTo $n$}{
    \For{$i=1$ \KwTo $n$}{
        \For{$j=1$ \KwTo $i-1$}{
            \If{ $\widehat{U}_{ij}\geq \widehat{U}_{ik}+ \widehat{U}_{jk}$}{
                    $\widehat{U}_{ij} = \min(\widehat{U}_{ij}, \widehat{U}_{ik}+\widehat{U}_{jk})$
                }
            } 
        } 
    }
$P = \widehat{U} - D'$
\label{alg-IOMROracle}
\end{algorithm}

 As with all previous algorithms, if we know the set of broken triangles of $U$ a priori, we can accelerate Algorithm~\ref{alg-IOMROracle} to an $O(|Q|n)$ algorithm.

\subsection{General Metric Repair}

Using the ideas in the previous subsections, we present an algorithm for metric repair in the general case. Our algorithm combines the ideas of the update steps for both DOMR and IOMR with the idea of finding an oracle. Essentially, the algorithm first uses the counting heuristic described in Section~\ref{sec:IOMR} to build two dictionaries, one which tracks the number of times that each distance occurs on the right or the left in a broken triangle. This information is then used to decide what distance to fix in a broken triangle. Intuitively, if $D'_{ij}$ occurs on the left of a broken triangle inequality often, this implies that we should decrease it. 

\begin{algorithm}
\newcommand{\lft}{{\rm{l}}}
\newcommand{\rght}{{\rm{r}}}
\SetAlgoLined
\caption{Heuristic for General Metric Repair}

\KwIn{Corrupted $n\times n$ distance matrix $D'$}
\KwResult{Perturbation $P$}
 $\widehat{D} = D$\\
Let $\lft, \rght$ be $n\times n$ symmetric all-zero matrices.\\
 \For{$k=1$ \KwTo $n$}{
    \For{$i=1$ \KwTo $n$}{
        \For{$j=1$ \KwTo $i-1$}{
            \If{ $\widehat{D}_{ij} \geq \widehat{D}_{ik}+\widehat{D}_{kj}$}{
                $\lft_{ij} = \lft_{ij}+1$
                $\rght_{ik} = \rght_{ik}+1$
                $\rght_{jk} = \rght_{jk}+1$
            } 
        } 
    }
 }
 
 \For{$k=1$ \KwTo $n$}{
    \For{$i=1$ \KwTo $n$}{
        \For{$j=1$ \KwTo $i-1$}{
            \If{ $\widehat{D}_{ij} \geq \widehat{D}_{ik}+\widehat{D}_{kj}$}{
                \lIf{$\lft_{ij} > \max(\rght_{ik}, \rght_{jk})$}{
                    $\widehat{D}_{ij} =\widehat{D}_{ik}+\widehat{D}_{kj}$                 
                }\lElseIf{$\rght_{ik} > \rght_{jk}$}{
                    $\widehat{D}_{ik} = \widehat{D}_{ij} - \widehat{D}_{jk}$
                }\Else {
                    $\widehat{D}_{jk} = \widehat{D}_{ij} - \widehat{D}_{ik}$                        
                } 
        } 
    }
 }
 }
 \label{alg-General}
 \end{algorithm}

In general, the heuristic algorithm is not guaranteed to succeed; i.e., the output may not be a fixed distance matrix. In addition, no guarantee is given for finding the sparsest solution; the heuristic constraints, however, certainly encourages sparsity. Again, the algorithm can be modified to loop over all triangles in $\T(D')$ and triangles containing a distance from $\T'(D')$.


\subsection{Random data}
Algorithm~\ref{alg-General} applies in situations where we do not have any prior knowledge of the sign of our perturbation. Although it is a heuristic, it is instructive to consider distances that are drawn at random specified distributions and examine the extent to which these random matrices fail to adhere to metricity so that when we demonstrate how the heuristic performs on this type of data in Section~\ref{sec:expts}, we have an understanding of its expected performance.

We specifically consider the model where $D'$ is a random symmetric matrix with entries independently drawn from a distribution. 

\begin{lemma}\label{lem:random}
The expected proportion of broken triangles in $D'$ is $1/6$ and $1/4$ when the entries of $D'$ are drawn from a uniform distribution $\text{Unif}([0,1])$, or an exponential distribution $\text{Exp}(\lambda)$.
\end{lemma}

In particular, this implies that if $n$ is sufficiently large, almost every distance will be involved in a broken triangle! Thus it seems unlikely to find an exceptionally sparse perturbation in all cases that will fix the metric. 

\subsection{Convex relaxation methods}
Finally, we finish with convex optimization methods (or, more colloquially, $\ell_1$ minimization). Using the standard paradigms from compressed sensingy, we relax the $\ell_0$ constraint on $P$ to a $\ell_1$ constraint:

\begin{equation}
\begin{aligned}
  &\text{minimize} & & \|P\|_1\\
  &\text{subject to} & & D'+P \text{ is metric}, P\in S\subset \mathbb{R}^{n\times n}
\end{aligned}
\label{alg-l1}
\end{equation}

We can also consider the metric constraints as polyhedral constraints on the set of symmetric $n\times n$ matrices (the resulting polyhedron is known as the metric cone, see for example \cite{deza2009encyclopedia}). So our sparse metric repair problem is asking for the sparsest $P$ so that $D'+P$ lies in this polyhedral cone. A commom technique used for finding the sparsest solution for a set of linear constraints is reweighted $\ell_1$ schemes \cite{needell2009noisy}. 

\begin{figure*}[t]
\centering
\includegraphics[trim={2cm 3cm 2cm 3cm},clip, width=.9\linewidth]{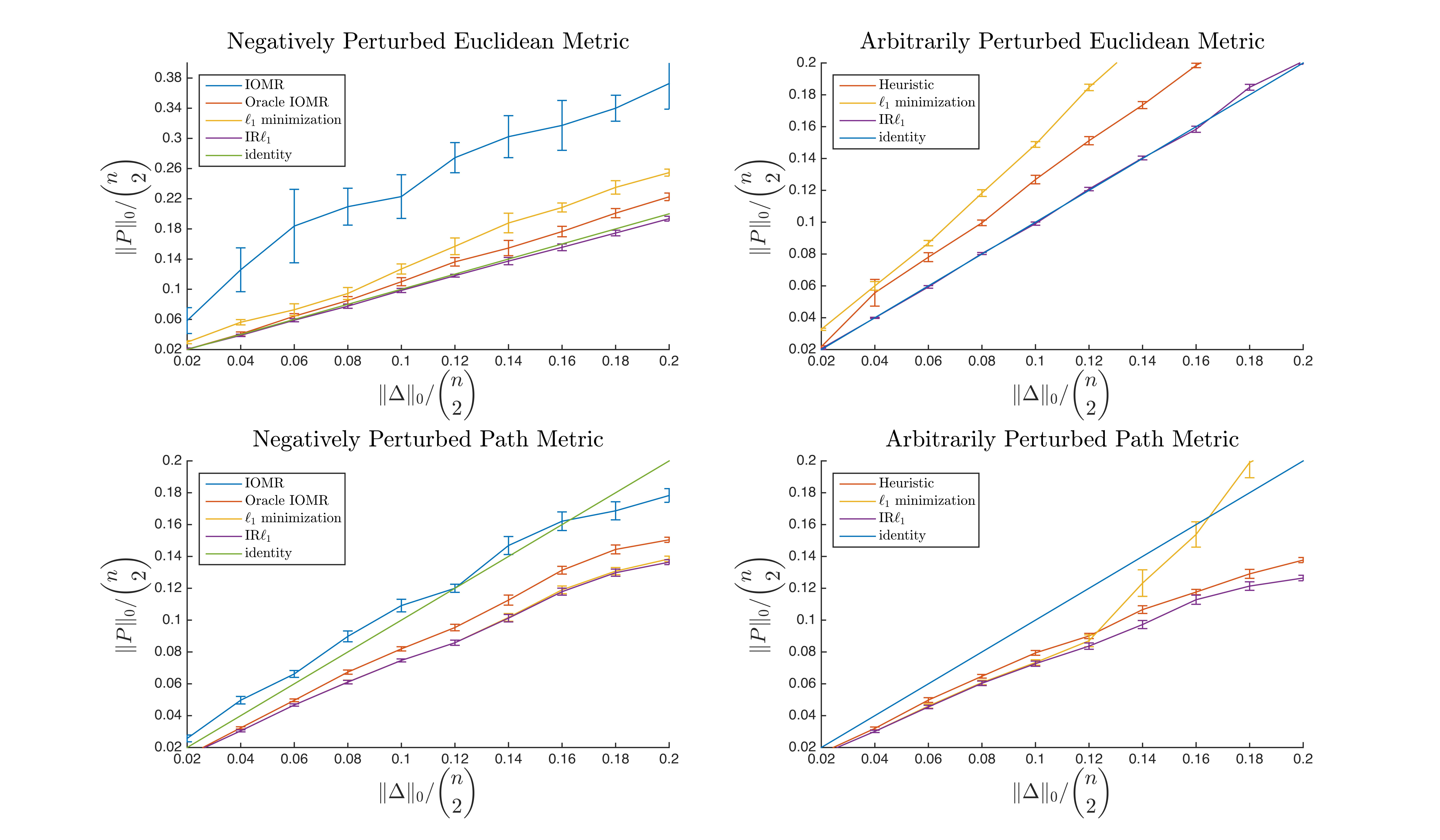}
\caption{Experiments on input vs output sparsity for the algorithms in this paper.}
\label{fig:sparsity}
\end{figure*}

\begin{algorithm}
\SetAlgoLined
\caption{Iteratively Reweighted $\ell_1$ minimization }

\KwIn{Corrupted $n\times n$ distance matrix $D$, iters, tolerance parameter $\epsilon$}
\KwResult{Fixed distance matrix $\widehat{D}$}

$w^1$,  $n\times n$ weight matrix with every entry initially 1.\\

\For{$t=1$ \KwTo \text{iters}} {
$P^t$ solution to following optimization problem
\begin{equation*}
\begin{aligned}
  &\text{minimize} & & \sum_{1\leq i,j\leq n}w_{ij}|P_{ij}|\\
  &\text{subject to} & & D'+P \text{ is metric}, P\in S\subset \mathbb{R}^{n\times n}
\end{aligned}
\end{equation*}

$w^{t+1}_{ij} = \frac{1}{P^{t-1}_{ij}+\epsilon}$
}
\label{alg-IRL1}
\end{algorithm}

In general reweighted schemes tend to find the sparsest solutions, however it can take many iterations for it to converge. Both algorithms implemented using standard solvers like CVX are computationally much more expensive compared to the discrete algorithms given above. In particular the best algorithms for checking that $D'+P$ is metric are essentially $O(n^3)$ (see below), so the convex algorithms have complexity at least as bad as the complexity of the discrete methods. 

\subsection{Identifying Broken Triangles}\label{sec-brokentriangles}
All of our algorithms have cubic run time and, with an input list of broken triangles, are much more efficient. To speed up these algorithms, it is sufficient to identify the broken (or negative) triangles efficiently. Indeed, this phenomenon is not a surprise as Williams and Williams~\cite{WilliamsWilliams2010} show that a number of problems are all equivalent under subcubic reductions, including negative triangle detection, APSP, and detecting metricity.  These reductions show that an $O(n^{3-\epsilon_1})$ algorithm for one problem would yield an $O(n^{3-\epsilon_2})$ algorithm for another.  As of 2017, there are, however, no truly subcubic algorithms for APSP. Chan and Williams~\cite{Chan2016} show that APSP on $n$-node weighted graphs can be solved in $O(n^{3/2^{\Omega(\sqrt{\log n})}})$ time deterministically. There are subcubic algorithms for graphs with integer weights and for unweighted graphs (for a comprehensive survey, see~\cite{madkour2017survey}). This \emph{might} apply in some, but not all, data applications. Finally, we note that Peres, et al.~\cite{Peres2010} show that on a random data model, APSP is a subcubic (even $O(n^2)$ algorithm) and we use this for inspiration in our experimental analysis.

\section{Experimental analysis}
\label{sec:expts}

Because the algorithms for DOMR are guaranteed to return the sparsest solution (see Section~\ref{sec:DOMR}), we focus on a comparison amongst our combinatorial algorithms, namely IOMR (Algorithm~\ref{alg-IOMR}), Oracle IOMR (Algorithm~\ref{alg-IOMROracle}), Heuristic for General Metric Repair (Algorithm~\ref{alg-General}) and the convex methods of $\ell_1$ minimization (i.e., the algorithmic formulation given by Equation~\ref{alg-l1}), and iteratively reweighted $\ell_1$ minimzation which we also denote by IR$\ell_1$(Algorithm~\ref{alg-IRL1}). 

In general, there are many different factors and models we can consider that impact the initial number of broken triangles and the resulting output sparsity. We could also analyze many aspects of the output, for example, the uniqueness of the sparsest perturbation, or how our solutions compare in magnitude. All of these cases are interesting to consider, but, for succinctness, we focus on sparsity of the output perturbation and the time it took the algorithm to run.

First, consider the case where we observe a broken metric $D'= D+\Delta$, with $D$ metric and $\Delta$ a sparse perturbation. We used the various algorithms to compute a perturbation $P$ so that $D'+P$ was again metric. In the top row of Figure~\ref{fig:sparsity}, $D$ is assumed to be the distance matrix of $50$ points in $\mathbb{R}^2$. The $x$-axis gives the sparsity of $\Delta$ as a fraction of the 1225 distances, and the $y$-axis is the sparsity level of $P$. In the left hand plot, $\Delta$ is assumed to be negative with values on the support uniformly chosen from $[-\|D\|_{\infty}/8, 0]$; on the right the values are uniformly chosen from $[-\|D\|_{\infty}/8, \|D\|_{\infty}/8]$. These values were chosen to be not too large, to assist finding $D$'s so that $0 \preceq D+\Delta$. In the bottom row, $D$ is the path metric of an Erd\"os-Renyi random graph on 50 vertices with $p = 2\log{50}/50$. Since the diameter of such Erd\"os-Renyi graphs is less than $3$ with high probability, the support of $\Delta$ was set to -1 in the negative perturbation case, and arbitrarily chosen from $\{-1,1\}$ in the arbitrary perturbation case. For each sparsity level, we averaged the sparsity of the output perturbation over ten trials. Finally, we also plot the the identity line as we expect that the sparsity of the output $P$ should not be too much greater than the sparsity of the input perturbation $\Delta$.

All algorithms were implemented in Matlab on a 2014 Macbook Pro running a 2.2 GHz Intel Core i7. We used the GNU Linear Programming Kit~\cite{GLPK} for the $\ell_1$ minimization based techinques. 

In all cases iteratively reweighted $\ell_1$ did the best, often finding a solution as sparse as the input perturbation or even sparser. When $\Delta\preceq 0$, IOMROracle tends to do almost as well as IR$\ell_1$, even matching it for low sparsity levels. In the Euclidean case, we see that IOMROracle and our Heuristic do significantly better than $\ell_1$. In the path case, $\ell_1$ minization and IR$\ell_1$ match for low sparsity levels, and both are slightly better than IOMROracle. We expect IOMR to be worse than IOMROracle since it is fixing distances in a predetermined order rather than a data-dependent fashion. It is surprising, however, that it does not seem to produce perturbations with sparsity worse than twice those of IOMROracle, so we conjecture that IOMR is potentially a 2-approximation algorithm for IOMR. For IOMROracle, we counted the broken triangles a distance occurred to decide whether we should fix it (similar to the method in the Heuristic algorithm). As a result, in several cases, IOMROracle and Heuristic both returned perturbations $P$ such that $D'+P$ had broken triangles. However, in each such case, at worst 20 triangles (around $1.6\%$ of all triangles) were broken---a significant improvement over the up to $60\%$ that were initially broken.  

\begin{figure}[h!]
\centering
\includegraphics[trim={.4cm .4cm .4cm .4cm},clip,width=.4\textwidth]{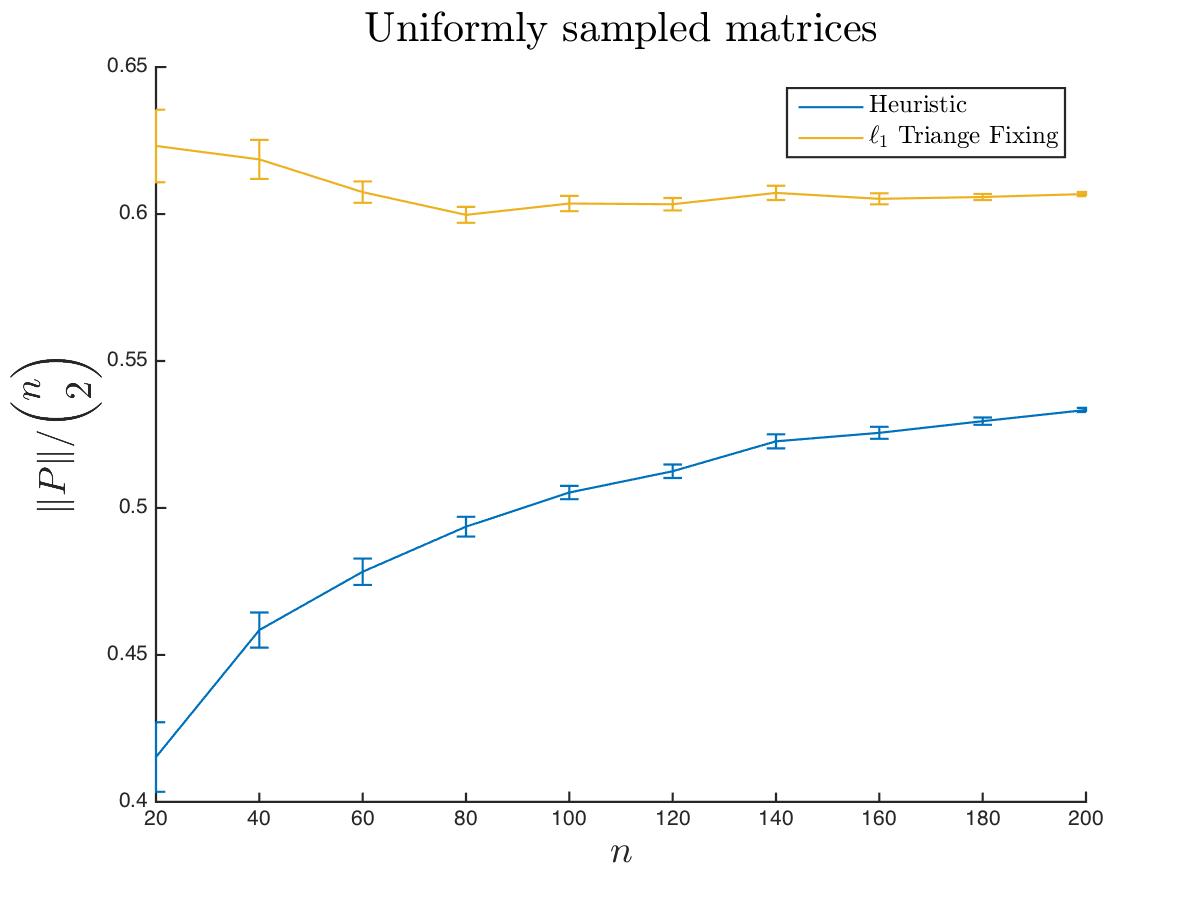}
\caption{Output sparsity of metrics with entries uniformly sampled from $[0,1]$.}
\label{fig:sparsity-uniform}
\end{figure}
The path metric case has several interesting attributes. Firstly, there is a strange divergence of $\ell_1$ minimization from IR$\ell_1$ in the arbitrary perturbation case for large sparsity cases. Secondly, in all the examples we ran, the both $\ell_1$ minimization and IR$\ell_1$ produced perturbations that were integer vectors! In addition the Heuristic almost always repaired the metric. The support of the solutions from the Heuristic and $\ell_1$ algorithms, differed from each other and $\Delta$.

We also considered the case when $D'$ was an arbitrary symmetric matrix whose entries were drawn uniformly from $[0,1]$. We compared the Heuristic method to the $\ell_1$ Triangle Fixing method given in \cite{Brickell}. Figure~\ref{fig:sparsity-uniform} shows the recovered sparsity as we vary $n$. As described in Lemma~\ref{lem:random}, we expect around $16\%$ of triangles to be broken, and almost all distances to be involved in a broken triangle. So, repairing or affecting only $50\%$ of distances is a significant improvement over adjusting all of them.

\begin{figure}[h!]
\centering
\includegraphics[trim={1.8cm .4cm 1.8cm .5cm},clip,width=2.5in]{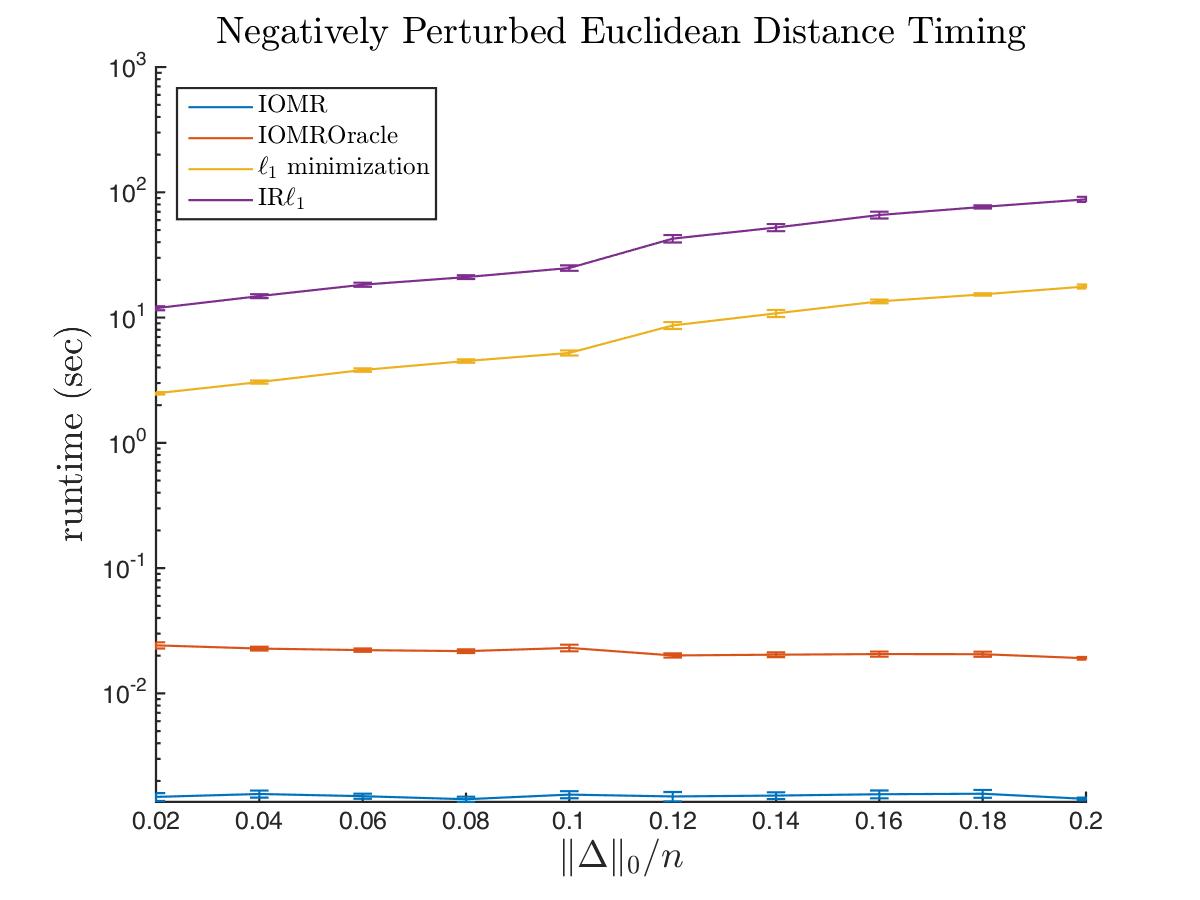}\\
\includegraphics[trim={1.8cm .5cm 1.8cm .3cm},clip,width=2.5in]{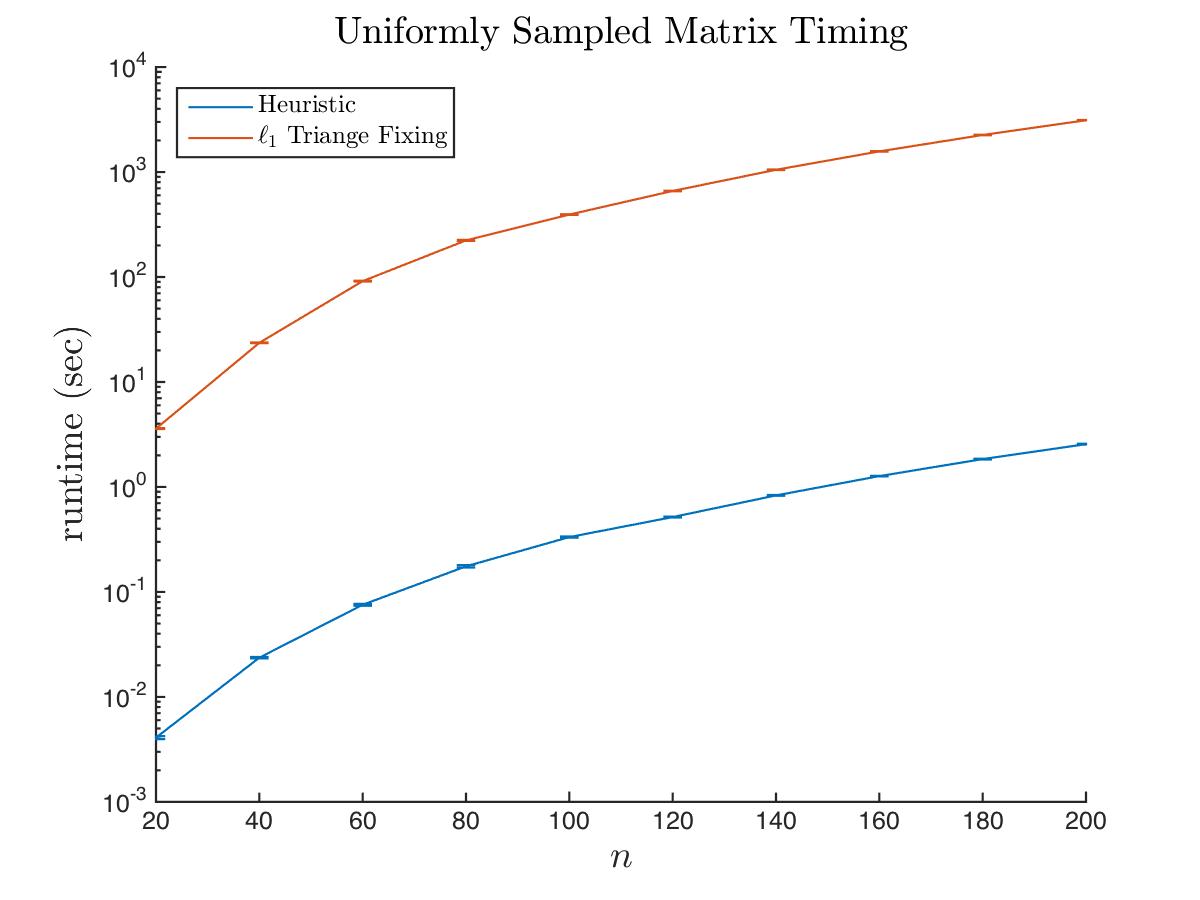}
\caption{Runtime comparison of the algorithms.}
\label{fig:runtime}
\end{figure}
The real advantage of the combinatorial algorithms over the convex methods is in their runtime; often several orders of magnitude. Figure~\ref{fig:runtime} shows the difference in runtime for the arbitrary perturbation Euclidean case. We also show the difference in timing for the Heuristic method vs Triangle Fixing. In general we found Triangle Fixing difficult to choose parameters for, and we also found that the number of passes over all triangles it required seemed to grow with $n$, unlike our Heuristic method which requires two. We realize that more specialized LP solvers could lead to large improvements and hope to address this in the future.


\section{Conclusions}
Our experiments show that the convex relaxation approaches to these problems have interesting empirical performance (depending on the input sparsity level and instance type). In addition, we can easily construct simple examples that have multiple (even infinite) solutions, all with the same minimal $\ell_1$ norm. These two observations suggest a deeper study of convex relaxation based methods. We emphasize that this is in contrast with the efficacy of convex relaxation for more traditional sparse problems. 

Many graph and network flow problems, such as Single Source Shortest Path, can be expressed as linear programs. Studying the dual of these programs has been fruitful in developing efficient combinatorial algorithms. In our work, we propose several different combinatorial algorithms and it would be interesting to ascertain what dual problem these algorithms solve. In particular it is likely that the IOMR algorithm is likely to correspond to the dual to a network flow problem extending the correspondence Brickell et al.~\cite{Brickell} observe for the DOMR case.

\addtolength{\textheight}{-12cm}   




\section*{ACKNOWLEDGMENT}

We would like to thank Sergey Fomin and Seth Pettie for several useful discussions about this problem and its connections to all pairs shortest path algorithms and tropical matrix multiplication.

\bibliographystyle{IEEEtran}
\bibliography{IEEEabrv,references}

\begin{thebibliography}{10}
\providecommand{\url}[1]{#1}
\csname url@rmstyle\endcsname
\providecommand{\newblock}{\relax}
\providecommand{\bibinfo}[2]{#2}
\providecommand\BIBentrySTDinterwordspacing{\spaceskip=0pt\relax}
\providecommand\BIBentryALTinterwordstretchfactor{4}
\providecommand\BIBentryALTinterwordspacing{\spaceskip=\fontdimen2\font plus
\BIBentryALTinterwordstretchfactor\fontdimen3\font minus
  \fontdimen4\font\relax}
\providecommand\BIBforeignlanguage[2]{{%
\expandafter\ifx\csname l@#1\endcsname\relax
\typeout{** WARNING: IEEEtran.bst: No hyphenation pattern has been}%
\typeout{** loaded for the language `#1'. Using the pattern for}%
\typeout{** the default language instead.}%
\else
\language=\csname l@#1\endcsname
\fi
#2}}

\bibitem{Xing2002}
E.~P. Xing, A.~Y. Ng, M.~I. Jordan, and S.~Russell, ``{Distance Metric Learning
  , with Application to Clustering with Side-Information},'' \emph{Advances in
  Neural Information Processing Systems}, vol.~15, pp. 505--512, 2002.

\bibitem{chen2009learning}
Y.~Chen, M.~R. Gupta, and B.~Recht, ``Learning kernels from indefinite
  similarities,'' in \emph{Proceedings of the 26th Annual International
  Conference on Machine Learning}.\hskip 1em plus 0.5em minus 0.4em\relax ACM,
  2009, pp. 145--152.

\bibitem{Indyk:1999:STA:301250.301366}
\BIBentryALTinterwordspacing
P.~Indyk, ``Sublinear time algorithms for metric space problems,'' in
  \emph{Proceedings of the Thirty-first Annual ACM Symposium on Theory of
  Computing}, ser. STOC '99.\hskip 1em plus 0.5em minus 0.4em\relax New York,
  NY, USA: ACM, 1999, pp. 428--434. [Online]. Available:
  \url{http://doi.acm.org/10.1145/301250.301366}
\BIBentrySTDinterwordspacing

\bibitem{anfinsen2005making}
J.~Anfinsen, ``Making substitution matrices metric,'' Master's thesis,
  Institutt for datateknikk og informasjonsvitenskap, 2005.

\bibitem{laub2007inducing}
J.~Laub, K.-R. M{\"u}ller, F.~A. Wichmann, and J.~H. Macke, ``Inducing metric
  violations in human similarity judgements,'' in \emph{Advances in neural
  information processing systems}, 2007, pp. 777--784.

\bibitem{baraty2011impact}
S.~Baraty, D.~Simovici, and C.~Zara, ``The impact of triangular inequality
  violations on medoid-based clustering,'' \emph{Foundations of Intelligent
  Systems}, pp. 280--289, 2011.

\bibitem{Brickell}
J.~Brickell, I.~S. Dhillon, S.~Sra, and J.~A. Tropp, ``The metric nearness
  problem,'' \emph{SIAM Journal on Matrix Analysis and Applications}, vol.~30,
  no.~1, pp. 375--396, 2008.

\bibitem{biswas2015efficient}
A.~Biswas and D.~W. Jacobs, ``An efficient algorithm for learning distances
  that obey the triangle inequality.'' in \emph{BMVC}, 2015, pp. 10--1.

\bibitem{GowerLegendre86}
J.~C. Gower and P.~Legendre, ``Metric and euclidean properties of dissimilarity
  coefficients,'' \emph{Journal of Classification}, no.~3, pp. 5--48, 1986.

\bibitem{WilliamsWilliams2010}
V.~V. Williams and R.~Williams, ``Subcubic equivalences between path, matrix
  and triangle problems,'' in \emph{2010 IEEE 51st Annual Symposium on
  Foundations of Computer Science}, Oct. 2010, pp. 645--654.

\bibitem{deza2009encyclopedia}
M.~M. Deza and E.~Deza, ``Encyclopedia of distances,'' in \emph{Encyclopedia of
  Distances}.\hskip 1em plus 0.5em minus 0.4em\relax Springer, 2009, pp.
  1--583.

\bibitem{needell2009noisy}
D.~Needell, ``Noisy signal recovery via iterative reweighted l1-minimization,''
  in \emph{Signals, Systems and Computers, 2009 Conference Record of the
  Forty-Third Asilomar Conference on}.\hskip 1em plus 0.5em minus 0.4em\relax
  IEEE, 2009, pp. 113--117.

\bibitem{Chan2016}
T.~M. Chan and R.~Williams, ``{Deterministic APSP, Orthogonal Vectors, and
  More: Quickly Derandomizing Razborov-Smolensky},'' \emph{SIAM ACM Symposium
  on Discrete Algorithms (SODA'16)}, pp. 1246--1255, 2016.

\bibitem{madkour2017survey}
A.~Madkour, W.~G. Aref, F.~U. Rehman, M.~A. Rahman, and S.~Basalamah, ``A
  survey of shortest-path algorithms,'' \emph{arXiv preprint arXiv:1705.02044},
  2017.

\bibitem{Peres2010}
Y.~Peres, D.~Sotnikov, B.~Sudakov, and U.~Zwick, ``{All-pairs shortest paths in
  $O(n^2)$ time with high probability},'' in \emph{Proceedings of the Annual
  IEEE Symposium on Foundations of Computer Science, FOCS}, 2010, pp. 663--672.

\bibitem{GLPK}
``Gnu linear programming kit,'' \url{https://www.gnu.org/software/glpk}.

\end{thebibliography}

\end{document}